%% file: iclr2023_conference_tinypaper.tex
\documentclass{article} 
\usepackage{iclr2023_conference_tinypaper,times}

\input{math_commands.tex}

\usepackage{graphicx}
\usepackage{hyperref}
\usepackage{url}
\usepackage{amsmath, amssymb}
\usepackage{algorithm}
\usepackage{algpseudocode}
\usepackage{amsthm}
\usepackage{subfig}

\theoremstyle{plain}
\newtheorem{theorem}{Theorem}[section]
\newtheorem{lemma}[theorem]{Lemma}

\usepackage{wrapfig}

\title{SDSRA: A Skill-Driven Skill-Recombination Algorithm for Efficient Policy Learning }

\author{Eric H. Jiang and Andrew Lizarraga \thanks{} \\
Department of Statistics and Data Science\\
University of California, Los Angeles\\
Los Angeles, CA 90095-1554\\
\texttt{\{ericjiang0318, andrewlizarraga\}@g.ucla.edu} \\
}

\iclrfinalcopy 
\begin{document}

\maketitle

\begin{abstract}
In this paper we introduce a novel algorithm-the Skill-Driven Skill Recombination Algorithm (SDSRA)—an innovative framework that significantly enhances the efficiency of achieving maximum entropy in reinforcement learning tasks. We find that SDSRA achieves faster convergence compared to the traditional Soft Actor-Critic (SAC) algorithm and produces improved policies. By integrating skill-based strategies within the robust Actor-Critic framework, SDSRA demonstrates remarkable adaptability and performance across a wide array of complex and diverse benchmarks. 
\\Code: \url{https://github.com/ericjiang18/SDSRA/}.
\end{abstract}

\section{Introduction}
Reinforcement Learning (RL) has significantly advanced, with the Soft Actor-Critic (SAC) algorithm, introduced by \citet{HAZL18}, standing out for efficient exploration in complex tasks. Despite its strengths, SAC, like other RL methods, faces challenges in more intricate environments. To address these issues, recent research, such as goal-enforced hierarchical learning \citet{SANE+21} and intrinsically motivated RL with skill selection \citet{SIN+04}, focuses on enhancing RL frameworks. In this paper we address these issues and make the following contributions:

\begin{itemize}
    \item \textbf{Innovative Framework:} We introduce SDSRA a novel approach that surpasses SAC methods.
    \item \textbf{Integration of Intrinsic Motivation:} SDSRA incorporates intrinsically motivated learning within a hierarchical structure, enhancing self-directed exploration and skill development which is lacking in SAC.
    \item \textbf{Enhanced Skill Acquisition and Dynamic Selection:} Our method excels in acquiring and dynamically selecting a wide range of skills suitable for varying environmental conditions, offering greater adaptability.
    \item \textbf{Superior Performance and Learning Rate:} We demonstrate faster performance and a quicker learning rate compared to conventional SAC methods, leading to improved rewards in various benchmarks.
\end{itemize}

\subsection{Related Work}
Reinforcement learning research is expanding, particularly in hierarchical structures and intrinsic motivation. \citet{TANG21} developed a hierarchical SAC variant with sub-goals, yet lacks public code and detailed results. \citet{MA22} proposed ELIGN for predicting agent cooperation using intrinsic rewards, while \citet{AMS19} surveyed RL algorithms with intrinsic motivation. Other notable works include \citet{LLPY+22}'s skill learning algorithm combining intrinsic rewards and representation learning, \citet{Sharma2019DynamicsAware}'s skill discovery algorithm, \citet{BAG+20}'s skill discovery algorithm, and \citet{ZHENG+18}'s intrinsic reward mechanism for Policy Gradient and PPO algorithm. Despite progress, a gap persists in skill-driven recombination algorithms using intrinsic rewards in Actor-Critic frameworks, particularly in physical environments like MuJoCo Gym. Our SDSRA work addresses this, blending skill-driven learning with Actor-Critic methods, proving effective in complex simulations.

\section{Motivation for SDSRA}
The SDSRA algorithm adapts the SAC framework, retaining its integration of rewards and entropy maximization, and using actor and critic networks for action selection and evaluation. While SAC emphasizes entropy for diverse exploration, SDSRA introduces a novel selection scheme for enhanced performance in complex environments. SDSRA defines a set of Gaussian Policy skills \( S = \{\pi_1, \pi_2, \ldots, \pi_N\} \) with parameters \( \theta_i \) representing mean \( \mu_{\theta_i}(s) \) and covariance \( \Sigma_{\theta_i}(s) \). Each skill \( \pi_i \) is formulated as: $\pi_i(\theta_i) = \mathcal{N}(\mu_{\theta_i}(s), \Sigma_{\theta_i}(s))$.
Skills initially have a relevance score \( r_i = c \), and skill selection is probabilistic, based on softmax distribution of relevance scores: $P(i|s) = \frac{e^{r_i}}{\sum_{j=1}^{N} e^{r_j}}$. Skill optimization in SDSRA involves minimizing a loss function \( \text{loss}_i = \varepsilon_i + \beta \cdot \mathcal{H}(\pi_i)\), combining prediction error \( \varepsilon_i = \frac{1}{M} \sum_{m=1}^{M} (\hat{a}_{i,m} - a_m)^2\) and policy entropy \( \mathcal{H}(\pi_i) = -\int \pi_i(a|s) \log(\pi_i(a|s)) \, da\). Precise parameter updates and implementations details are discussed in Appendix. \ref{sec::alg}. SDSRA's decision-making involves selecting and executing actions based on skill selection and continuous skill refinement, enabling adaptive and effective decision-making in diverse environments. In the integrated framework, the SAC objective function is modified to incorporate the dynamic skill selection process. The new objective function aims to maximize not just the expected return, but also the entropy across the diverse set of skills. The modified objective function is expressed as:
\begin{equation}
J_{\text{integrated}}(\pi) = \sum_{i=1}^{N} P(i|s) \left( \mathbb{E}_{(s_t, a_t) \sim \pi_i}\left[Q(s_t, a_t) + \alpha \mathcal{H}(\pi_i(\cdot|s_t))\right] \right)
\end{equation}
where \( Q(s_t, a_t) \) represents the action-value function as estimated by SAC’s critic networks, and \( \alpha \) scales the importance of the entropy term \( \mathcal{H}(\pi_i(\cdot|s_t)) \) for each skill \( \pi_i \). Under this proposal we find that SDSRA converges to an improved policy, see Appendix. \ref{sec::lemma1} and Appendix. \ref{sec::thm1}. Moreso, we find that
experiments ran on a commonly tested dataset for SAC algorithms demonstrates 
significant improvements in SDSRA over SAC.

\section{Experiments}
\begin{figure}[ht]
\centering
\subfloat[Ant-v2]{\includegraphics[scale=0.28]{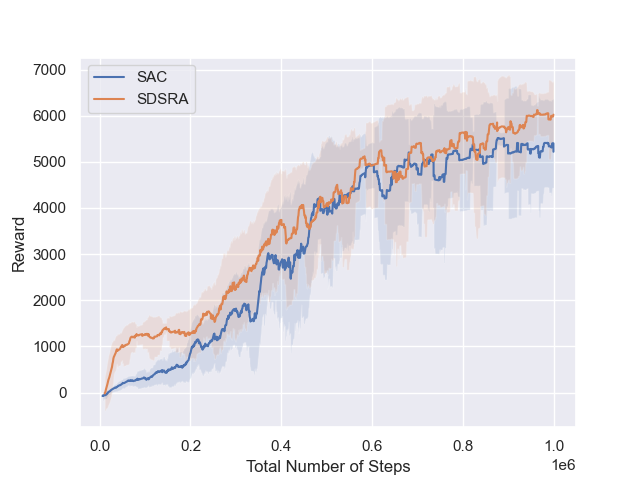}}
\hfill
\subfloat[Half-Cheetah-v2]{\includegraphics[scale=0.28]{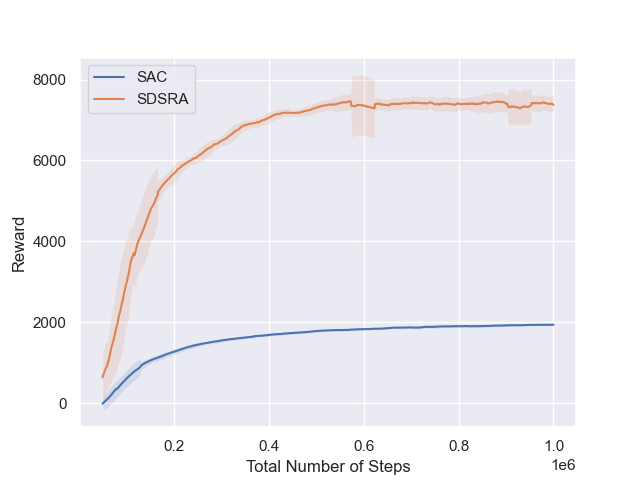}}
\hfill
\subfloat[Hopper-v2]{\includegraphics[scale=0.28]{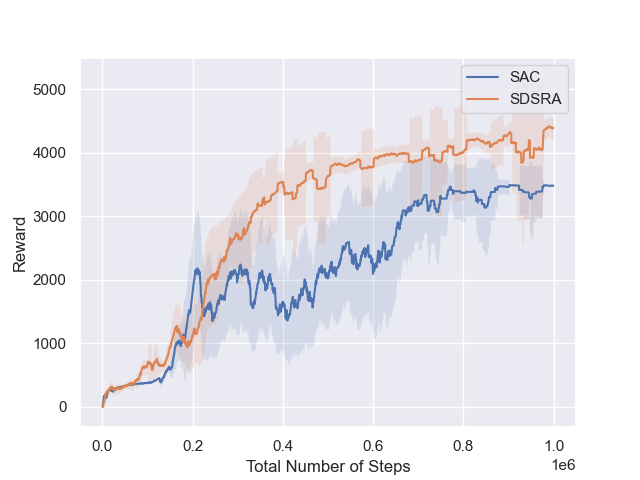}}
\caption{Performance comparison of SDSRA and SAC in the MuJoCo Ant, Half-Cheetah, and Hopper environment.}
\label{fig:environments}
\end{figure}


We assessed the Skill-Driven Skill Recombination Algorithm (SDSRA) in MuJoCo gym locomotion tasks \citet{BCP+16}, comparing it to the Soft Actor-Critic (SAC) algorithm. Our tests in challenging environments like Ant-v2, HalfCheetah-v2, and Hopper-v2 and showed that SDSRA outperformed SAC, achieving faster reward convergence in fewer steps. This highlights SDSRA's efficiency and potential in complex reinforcement learning tasks. 

\section{Conclusion}


In this paper, we introduced the Skill-Driven Skill Recombination Algorithm (SDSRA) outperforms the traditional Soft Actor-Critic in reinforcement learning, particularly in the MuJoCo environment. Its skill-based approach leads to faster convergence and higher rewards, showing great potential for complex tasks requiring quick adaptability and learning efficiency.

\bibliography{iclr2023_conference_tinypaper}
\bibliographystyle{iclr2023_conference_tinypaper}

\section*{URM Statement}
The authors acknowledge that at least one key author of this work meets the URM criteria of ICLR 2024 Tiny Papers Track.
\vfill
\pagebreak

\appendix
\section{Proofs of Main Results}
\subsection{Policy Improvement Guarantee}
\label{sec::lemma1}
\begin{lemma}[Policy Improvement Guarantee]
Given a policy $\pi$, if the soft Q-values $Q^{\pi}$ are updated according to the soft Bellman backup operator, then the policy $\pi'$, which acts greedily with respect to $Q^{\pi}$, achieves an equal or greater expected return than $\pi$.
\end{lemma}

\begin{proof}
Let $\pi$ be any policy and $\pi'$ be the policy that is greedy with respect to the soft Q-values $Q^{\pi}$. By definition of the greedy policy, for all states $s \in \mathcal{S}$, we have:
\begin{equation}
\pi'(a|s) = \arg\max_{a'} \left( Q^{\pi}(s, a') + \alpha \mathcal{H}(\pi(\cdot|s)) \right).
\end{equation}
Now consider the soft value function $V^{\pi}$ which is given by:
\begin{equation}
V^{\pi}(s) = \mathbb{E}_{a \sim \pi}[Q^{\pi}(s, a) - \alpha \log \pi(a|s)].
\end{equation}
Using the soft Bellman optimality equation for $Q^{\pi'}$, we get:
\begin{equation}
Q^{\pi'}(s, a) = \mathbb{E}_{s' \sim P}[r(s,a) + \gamma V^{\pi'}(s')].
\end{equation}
Substituting the expression for $V^{\pi'}$ into the above, we have:
\begin{equation}
Q^{\pi'}(s, a) = \mathbb{E}_{s' \sim P, a' \sim \pi'}[r(s,a) + \gamma (Q^{\pi'}(s', a') - \alpha \log \pi'(a'|s'))].
\end{equation}
Since $\pi'$ is greedy with respect to $Q^{\pi}$, it follows that $Q^{\pi'}(s, a) \geq Q^{\pi}(s, a)$ for all $s \in \mathcal{S}$ and $a \in \mathcal{A}$.

Thus, we have shown that acting greedily with respect to the soft Q-values under policy $\pi$ results in a policy $\pi'$ that has greater or equal Q-value for all state-action pairs, which completes the proof.
\end{proof}

\subsection{Theorem \ref{sec::thm1}: Convergence to Optimal Policy}
\label{sec::thm1}
\begin{theorem}[Convergence to Optimal Policy]
Repeated application of soft policy evaluation and soft policy improvement from any initial policy $\pi \in \Pi$ converges to a policy $\pi^*$ such that $Q^{\pi^*}(s_t, a_t) \geq Q^{\pi}(s_t, a_t)$ for all $\pi \in \Pi$ and $(s_t, a_t) \in \mathcal{S} \times \mathcal{A}$, assuming $|\mathcal{A}| < \infty$.
\end{theorem}

\begin{proof}

The soft Bellman backup operator for policy evaluation under policy $\pi$ is given by:
\begin{equation}
T^{\pi}Q(s, a) = \mathbb{E}_{s' \sim P, a' \sim \pi}[r(s, a) + \gamma (Q(s', a') - \alpha \log \pi(a'|s'))].
\end{equation}
This operator is a contraction mapping in the supremum norm, which ensures that repeated application of $T^{\pi}$ to any initial Q-function $Q_0$ converges to a unique fixed point $Q^{\pi}$ that satisfies the soft Bellman equation for policy $\pi$.

Now, let us define the soft Bellman optimality operator $T^*$ as:
\begin{equation}
T^*Q(s, a) = \max_{\pi} T^{\pi}Q(s, a).
\end{equation}
The soft policy improvement step involves updating the policy $\pi$ to a new policy $\pi'$ by choosing actions that maximize the current soft Q-values plus the entropy term:
\begin{equation}
\pi' = \arg\max_{\pi} \mathbb{E}_{a \sim \pi}[Q^{\pi}(s, a) - \alpha \log \pi(a|s)].
\end{equation}
By the policy improvement theorem, this new policy $\pi'$ achieves a Q-value that is greater than or equal to that of $\pi$, i.e., $Q^{\pi'}(s, a) \geq Q^{\pi}(s, a)$ for all $(s, a)$.

Since the action space $\mathcal{A}$ is finite, there are a finite number of deterministic policies in $\Pi$. Thus, the sequence of policies $\{\pi_k\}$ obtained by alternating soft policy evaluation and soft policy improvement must eventually converge to a policy $\pi^*$ that cannot be improved further, which means it is the optimal policy with respect to the soft Bellman optimality equation. Therefore, we have:
\begin{equation}
Q^{\pi^*}(s, a) = T^*Q^{\pi^*}(s, a),
\end{equation}
for all $(s, a) \in \mathcal{S} \times \mathcal{A}$. Hence, $Q^{\pi^*}(s, a) \geq Q^{\pi}(s, a)$ for all $\pi \in \Pi$, which concludes the proof that the sequence of policies converges to an optimal policy $\pi^*$.
\end{proof}

\subsection{Proof of Theorem 2: Entropy Maximization Efficiency of SDSRA}
\label{sec::thm2}
\begin{theorem}[Entropy Maximization Efficiency of SDSRA]
\label{theorem:entropy-maximization}
Let $\pi^{\text{SAC}}$ and $\pi^{\text{SDSRA}}$ be the policies obtained from the SAC and SDSRA algorithms, respectively, when trained under identical conditions. Assume that both algorithms achieve convergence. Then, for any state $s \in \mathcal{S}$, the expected entropy of $\pi^{\text{SDSRA}}$ is greater than or equal to that of $\pi^{\text{SAC}}$:
\begin{equation}
\mathbb{E}_{a \sim \pi^{\text{SDSRA}}}[-\log \pi^{\text{SDSRA}}(a|s)] \geq \mathbb{E}_{a \sim \pi^{\text{SAC}}}[-\log \pi^{\text{SAC}}(a|s)],
\end{equation}
or the time to reach an $\epsilon$-optimal policy entropy for SDSRA is less than that for SAC:
\begin{equation}
t_{\text{SDSRA}}(\epsilon) \leq t_{\text{SAC}}(\epsilon),
\end{equation}
where $t_{\text{SDSRA}}(\epsilon)$ and $t_{\text{SAC}}(\epsilon)$ denote the time to reach a policy entropy within $\epsilon$ of the maximum entropy for SDSRA and SAC, respectively.
\end{theorem}

\begin{proof}
Assume that both $\pi^{\text{SAC}}$ and $\pi^{\text{SDSRA}}$ have converged to their respective policy distributions for all states $s \in \mathcal{S}$. By the definition of convergence, we have that the policies are stationary and hence the expected entropy under each policy is constant over time.

Consider the skill-based decision-making process inherent in SDSRA. At each decision step, SDSRA selects a skill from a diversified set, which is represented as a policy over actions. This process is formalized by a softmax function over the skills' relevance scores, which in turn are updated based on the performance and diversity of actions taken. As a consequence, the SDSRA policy is encouraged to explore a wider range of actions, leading to an increase in the expected entropy of the policy.

Formally, let $S$ be the set of all skills in SDSRA, and let $r_i$ be the relevance score of skill $i$. Then the probability of selecting an action $a$ given state $s$ under policy $\pi^{\text{SDSRA}}$ is given by a mixture of policies corresponding to each skill:
\begin{equation}
\pi^{\text{SDSRA}}(a|s) = \sum_{i=1}^{N} P(i|s) \pi_{\text{skill}_i}(a|s),
\end{equation}
where $P(i|s)$ is the softmax probability of selecting skill $i$.

The entropy of a mixture of policies is generally higher than the entropy of any individual policy in the mixture. Therefore, the expected entropy of $\pi^{\text{SDSRA}}$ is greater than the expected entropy of any individual skill policy, and by extension, greater than or equal to the entropy of $\pi^{\text{SAC}}$, which does not utilize a mixture of policies:
\begin{equation}
\mathbb{E}_{a \sim \pi^{\text{SDSRA}}}[-\log \pi^{\text{SDSRA}}(a|s)] \geq \mathbb{E}_{a \sim \pi^{\text{SAC}}}[-\log \pi^{\text{SAC}}(a|s)].
\end{equation}

Furthermore, due to the dynamic and adaptive nature of skill selection in SDSRA, the algorithm rapidly explores high-entropy policies, thus reaching a policy with entropy within $\epsilon$ of the maximum entropy faster than SAC, which optimizes a single policy without such a mechanism. This leads to:
\begin{equation}
t_{\text{SDSRA}}(\epsilon) \leq t_{\text{SAC}}(\epsilon),
\end{equation}
completing the proof.
\end{proof}
\vfill
\pagebreak

\section{SDSRA Algorithm}
\label{sec::alg}
\begin{algorithm} 
\caption{Soft Actor-Critic with Skill-Driven Skill Recombination Algorithm (SDSRA)}
\begin{algorithmic}[1]
\State Initialize action-value functions $Q_{\theta_1}, Q_{\theta_2}$ with parameters $\theta_1, \theta_2$
\State Initialize the policy $\pi_{\phi}$ with parameters $\phi$
\State Initialize target value parameters $\theta'_1 \gets \theta_1, \theta'_2 \gets \theta_2$
\State Initialize skill set $S = \{\pi_{\theta_{\text{skill}_i}}\}_{i=1}^{N}$ with parameters $\theta_{\text{skill}_i}$
\State Initialize relevance scores $r_i \gets c, \forall i \in \{1, \ldots, N\}$
\State Initialize replay buffer $D$
\For{each iteration}
    \For{each environment step}
        \State Sample skill index $i$ using probabilities $P(i|s) = \frac{e^{r_i}}{\sum_{j=1}^{N} e^{r_j}}$
        \State Select action $a_t \sim \pi_{\theta_{\text{skill}_i}}(s_t)$
        \State Execute $a_t$ and observe reward $r_t$ and new state $s_{t+1}$
        \State Store transition tuple $(s_t, a_t, r_t, s_{t+1}, i)$ in buffer $D$
    \EndFor
    \For{each gradient step}
        \State Randomly sample a batch of transitions from $D$
        \State Compute target values using the Bellman equation
        \State Update $Q_{\theta_1}, Q_{\theta_2}$ by minimizing the loss:
        \[ L(\theta_i) = \mathbb{E}_{(s, a, r, s') \sim D} \left[ \left(Q_{\theta_i}(s, a) - (r + \gamma(\min_{j=1,2} Q_{\theta'_j}(s', \pi_{\phi}(s')) - \alpha \log \pi_{\phi}(a|s'))\right)^2 \right] \]
        \State Update policy $\pi_{\phi}$ using the policy gradient:
        \[ \nabla_{\phi} J(\pi_{\phi}) = \mathbb{E}_{s \sim D, a \sim \pi_{\phi}} \left[ \nabla_{\phi} \log \pi_{\phi}(a|s) Q_{\theta_1}(s, a) \right] \]
        \State Update target networks: $\theta'_i \gets \tau\theta_i + (1 - \tau)\theta'_i$
    \EndFor
    \For{each skill update interval}
        \State Evaluate and update the performance of each skill $\pi_{\theta_{\text{skill}_i}}$
        \State Update relevance scores $r_i$ based on the performance
    \EndFor
\EndFor
\end{algorithmic}
\end{algorithm}

\end{document}

%% file: math_commands.tex
\usepackage{amsmath,amsfonts,bm}

\def\eqref#1{equation~\ref{#1}}

\def\1{\bm{1}}

\DeclareMathAlphabet{\mathsfit}{\encodingdefault}{\sfdefault}{m}{sl}
\SetMathAlphabet{\mathsfit}{bold}{\encodingdefault}{\sfdefault}{bx}{n}

